\documentclass{article}

 \usepackage[final]{neurips_2025_ml4ps}

\usepackage[utf8]{inputenc} 
\usepackage[T1]{fontenc}    
\usepackage{hyperref}       
\usepackage{url}            
\usepackage{booktabs}       
\usepackage{amsfonts}       
\usepackage{nicefrac}       
\usepackage{microtype}      
\usepackage{xcolor} 

\usepackage{amsmath, amssymb, amsthm}
\usepackage{bm}            
\usepackage{graphicx}      
\usepackage{hyperref}      
\usepackage{enumerate}     
\usepackage{algorithm}     
\usepackage{algpseudocode} 
\usepackage{subcaption}    
\usepackage{xcolor}        
\usepackage{thmtools}      
\usepackage{cleveref}      

\declaretheorem[name=Theorem,numberwithin=section]{theorem}

\declaretheorem[name=Proposition,numberlike=theorem]{proposition}

\numberwithin{equation}{section}

\usepackage{wrapfig}
\title{Latent Target Score Matching, with an application to Simulation-Based Inference}

\author{%
  Joohwan Ko\\
  University of Massachusetts Amherst
  \\
  \texttt{joohwanko@cs.umass.edu} \\
  \And
  Tomas Geffner
 \\
  NVIDIA \\
  \texttt{tgeffner@nvidia.com} \\
}

\begin{document}

\maketitle

\begin{abstract}
Denoising score matching (DSM) for training diffusion models may suffer from high variance at low noise levels. Target Score Matching (TSM) mitigates this when clean data scores are available, providing a low-variance objective. In many applications clean scores are inaccessible due to the presence of latent variables, leaving only joint signals exposed. We propose Latent Target Score Matching (LTSM), an extension of TSM to leverage joint scores for low-variance supervision of the marginal score. While LTSM is effective at low noise levels, a mixture with DSM ensures robustness across noise scales. Across simulation-based inference tasks, LTSM consistently improves variance, score accuracy, and sample quality.\looseness=-1
\end{abstract}

\section{Introduction}

Diffusion models trained with the denoising score matching (DSM) objective \citep{vincent2011connection} have emerged as a powerful class of generative models \citep{ho2020denoising,song2019generative,song2021scorebased,sohl2015deep,karras2022elucidating}. These models learn the score, the gradient of the log-density of the diffused data distribution, to parameterize a generative process for sample generation. While widely applicable \citep{vincent2011connection,song2019generative}, the DSM objective may suffer from large variance as the diffusion noise level approaches zero \citep{de2024target}, potentially  degrading accuracy and sample quality. Target Score Matching (TSM) \cite{de2024target} addresses this, providing a low-variance training target for low noise levels in settings where clean data scores are accessible.
Many problems of interest, however, involve latent variables \citep{brehmer2020mining}. Often, we are interested in modeling a subset of variables while treating the rest as nuisance or auxiliary components. Examples include coarse-graining in structural biology \citep{marrink2007martini,arts2023two}, backbone protein design \citep{watson2023novo,wu2024protein,ingraham2023illuminating,geffner2025proteina}, or inference tasks where only a few parameters are of interest \citep{brehmer2020mining}. In such cases, simulators or models may provide additional signals, such as joint density values and scores, which are unused by DSM and TSM \citep{brehmer2020mining,pmlr-v119-hermans20a}.\looseness=-1

We propose Latent Target Score Matching (LTSM), a diffusion training objective tailored to models with latent variables. LTSM extends TSM by showing that the marginal score of interest (over non-latent variables) can be expressed as the conditional expectation of joint scores involving the latents. This leads to low-variance regression targets in the low-noise regime, directly addressing the shortcomings of DSM. Furthermore, LTSM can be combined with DSM through a simple time-dependent mixture, which balances accuracy near zero noise with robustness at larger noise levels. We evaluate LTSM on several simulation-based inference \citep{tejero2020sbi} tasks, where “gray-box” models expose joint information over parameters, latents, and observations \citep{brehmer2020mining}. We observe that the mixture objective that combines DSM and LTSM consistently improves over DSM in terms of score accuracy and sample quality.\looseness=-1

\section{Preliminaries}

\textbf{Diffusion (VP-SDE).} Consider the variance preserving~\citep{song2021scorebased} stochastic differential equation (SDE)
\begin{align} \label{eq:fwd}
    \mathrm{d}\theta_t \;=\; -\tfrac{1}{2}\beta(t)\,\theta_t\,\mathrm{d}t + \sqrt{\beta(t)}\,\mathrm{d}W_t, \quad \theta_0 \sim p_\mathrm{data}(\theta),
\end{align}
for $t\in[0, 1]$ and \(W_t\) a standard Wiener process.
Intuitively, this process gradually diffuses the data distribution towards a standard Gaussian. 
Given $\theta_0$ and $\beta(t)\geq 0$, \cref{eq:fwd} can be simulated exactly as
$\theta_t \sim p_t(\theta_t \mid \theta_0)
= \mathcal{N}\!(\theta_t;\,\alpha(t)\,\theta_0,\,(1-\alpha^2(t))\,I)$, where
$\alpha(t)=\exp(-\tfrac12\int_0^t \beta(s)\,ds).$

Diffusion models can be used to sample $p_\mathrm{data}(\theta)$ by simulating the time-reversal of \cref{eq:fwd}\looseness=-1
\begin{equation} \label{eq:bwd}
\mathrm{d}\theta_t \;=\; \left[-\tfrac{1}{2}\beta(t)\,\theta_t \;-\; \beta(t)\,\nabla_{\theta}\log p_t(\theta_t)\right]\mathrm{d}t \;+\; \sqrt{\beta(t)}\,\mathrm{d}\bar W_t, \quad \theta_1 \sim N(0, I),
\end{equation}
from $t=1$ to $t=0$. This requires the score \(\nabla_\theta\log p_t \left(\theta_t\right)\), which is typically intractable. Diffusion models approximate it using a neural network, typically trained with the DSM loss
\citep{hyvarinen2005estimation,vincent2011connection}\looseness=-1
\begin{equation}
\label{eq:dsm-loss}
\mathcal{L}_{\text{DSM}}(\psi)
=
\int_0^1
\mathbb{E}_{\theta_0,\theta_t\vert\theta_0}
\Big[
  \lambda(t)\,
  \big\|
    s_\psi(\theta_t,t)
    - \underbrace{\nabla_{\theta_t}\log p_t(\theta_t\mid\theta_0)}_{y_\mathrm{DSM}(\theta_0, \theta_t, t) \text{ (regression target)}}
  \big\|^2
\Big]\mathrm{d}t.
\end{equation}
Since the DSM regression target \(y_\mathrm{DSM}(\theta_0, \theta_t, t)\) is an unbiased estimator of the true score, i.e., $\nabla_{\theta_t}\log p_t(\theta_t) = \mathbb{E}_{\theta_0 \vert \theta_t}\left[ \nabla{\theta_t}\log p_t(\theta_t\mid \theta_0) \right]$, minimizing the DSM loss trains $s_\psi$ to approximate the true score. For the VP-SDE, the regression target from eq.~\ref{eq:dsm-loss} is given by \( \nabla_{\theta_t}\log p_t(\theta_t\mid \theta_0) = (\alpha(t)\theta_0-\theta_t)/(1-\alpha(t)^2) \), with \( \alpha(t)=\exp \left[-\nicefrac{1}{2}\left(\beta_{\min}t+\frac{1}{2}(\beta_{\max}-\beta_{\min})t^2\right)\right] \) for a standard linear schedule \( \beta(t)=\beta_{\min}+(\beta_{\max} -\beta_{\min}) t \). The main drawback of this estimator is its high variance as $t \to 0$, a problem that motivated alternative estimators with better performance for small $t$.

\textbf{Target Score Matching \citep{de2024target}.}
TSM is a method designed for settings where the score of the clean data \(\nabla_{\theta_0}\log p(\theta_0)\) is known. It follows the same core principle as DSM, training a network \(s_\psi\) by regressing against an unbiased estimator of the true score \(\nabla_{\theta_t}\log p_t(\theta_t)\). The key innovation in TSM is the introduction of a different unbiased estimator, designed to have low variance as \(t \to 0\)\looseness=-1
\begin{equation} \mathcal{L}_{\text{TSM}}(\psi) = 
\int_0^1
\mathbb{E}_{\theta_0,\theta_t\vert\theta_0}
\Big[
\lambda(t)
\big\Vert
s_\psi(\theta_t,t) - \underbrace{\tfrac{1}{\alpha(t)} \nabla_{\theta_0}\log p(\theta_0)}_{y_\mathrm{TSM}(\theta_0, t) \text{ (regression target)}}
\big\Vert^2
\Big]\mathrm{d}t.
\end{equation}

\section{Latent Target Score Matching}
\label{sec:ltsm}

\begin{figure}[t]
    \centering
    \includegraphics[width=1.0\linewidth]{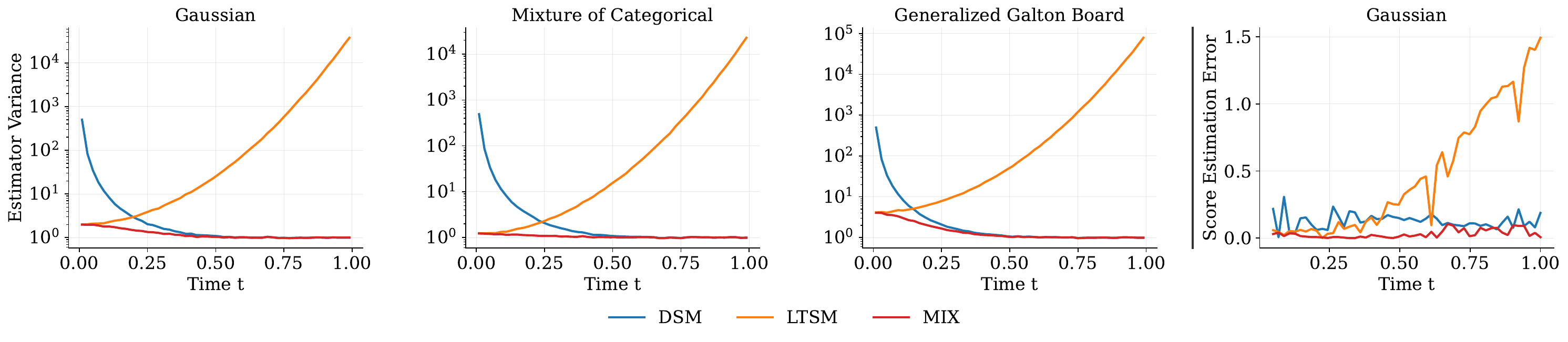}
    \caption{\small [Left] Regression target variance vs.\ diffusion time $t$ for DSM and LTSM across three tasks (\cref{sec:exps}). DSM rises as $t\to0$, while LTSM stays low at small $t$, increasing for larger $t$. The time-dependent mixture retains the best of both methods. [Right] Score estimation error, the DSM+LTSM mixture yields the best results.}
    \label{fig:variance}
\end{figure}

We consider a class of probabilistic models defined by a joint density $p(\theta, z)$, where \textbf{$\theta$} represents the variables of interest and $z$ represents auxiliary variables, which may be latent or simply nuisance components that we do not model directly. This structure is prevalent across domains, such as structural biology (coarse-graining \citep{marrink2007martini,arts2023two}, backbone protein design \citep{watson2023novo,wu2024protein,ingraham2023illuminating}), and simulation-based inference (SBI) \citep{brehmer2020mining}, among others. Our goal is to train a diffusion model to sample the marginal $p(\theta)$, assuming access to samples from $p(\theta, z)$ and the ability to evaluate the joint score $\nabla_{\theta} \log p(\theta, z)$.
While DSM can be applied, as it only requires samples of $\theta$, it ignores the information in the joint score and may suffer from high variance as $t \to 0$ \citep{de2024target}. TSM was designed to address this issue, but is not applicable, as it requires the clean marginal score $\nabla_{\theta_0} \log p(\theta_0)$, which is often intractable for latent-variable models. This motivates our development of Latent Target Score Matching (LTSM), a training objective that leverages the joint score to provide low-variance supervision for the marginal score.
LTSM extends TSM to the latent variable setting by showing that the joint score can be used to build an unbiased estimator of the marginal score of interest:\looseness=-1

\begin{proposition}[Latent Target Score Identity]
Under the VP-SDE that diffuses $\theta$ (and not $z$), we have\looseness=-1
\begin{equation}
\nabla_{\theta_t} \log p_t(\theta_t) = \frac{1}{\alpha(t)} \mathbb{E}_{\theta_0, z \vert \theta_t}[\nabla_{\theta_0} \log p(\theta_0, z)].
\end{equation}
\end{proposition}

This identity provides a new unbiased estimator for the marginal score of interest learned by diffusion models $\nabla_{\theta_t} \log p_t(\theta_t)$, which can be directly used to derive the LTSM training objective
\begin{equation} \label{eq:loss-ltsm}
\mathcal{L}_{\text{LTSM}}(\psi) = \int_{0}^{1} \mathbb{E}_{p(\theta_0, z) p_t(\theta_t|\theta_0)} \big[\eta(t) \|s_\psi(\theta_t, t) - \underbrace{\tfrac{1}{\alpha(t)} \nabla_{\theta_0} \log p(\theta_0, z)}_{y_\mathrm{LTSM}(\theta_0, z, t) \text{ (regression target)}}\|^2\big] d\mathrm{t},
\end{equation}
 where we denote the LTSM regression target as \(y_\mathrm{LTSM}(\theta_0, z, t)\). As intended, the LTSM target remains well-conditioned for small $t$, as shown in \cref{fig:variance}. However, \cref{fig:variance} also shows that its variance increases for larger $t$, where DSM is often more stable. This complementary behavior, also noted in the original TSM work \citep{de2024target}, motivates a new regression target obtained as an affine combination of the DSM and LTSM targets. We define the mixture regression target as 
 \begin{equation}
 y_{\text{MIX}}(\theta_0, z,\theta_t, t;w_t) = w_t y_{\text{DSM}}(\theta_0, \theta_t, t) + (1 - w_t) y_{\text{LTSM}}(\theta_0, z, t),
 \label{eq:y_mix}
 \end{equation}
 where a time-dependent weight $w_t \in \mathbb{R}$ is used to mix both targets. Since $y_{\text{MIX}}$ is an unbiased estimator of the true score for any $w_t$, it can be used to obtain a new training objective
\begin{equation} \label{eq:loss-mix}
\mathcal{L}_{\mathrm{MIX}}(\psi)
=
\int_0^1
\mathbb{E}\Big[\eta(t)
  \big\lVert s_\psi(\theta_t,t) - y_{\text{MIX}}(\theta_0, z,\theta_t, t;w_t)\big\rVert^2
\Big]\mathrm{d}t.
\end{equation}
Intuitively, choosing $w_t$ trades variance optimally across time. While its value can be chosen heuristically, we show the $w_t$ that minimizes the variance of $y_{\text{mix}}$ can be computed analytically:
\begin{proposition}[Optimal mixture weight]
\label{prop:optimal-w}
Define the mixture regression target $y_\mathrm{MIX}$ as in \cref{eq:y_mix}. Then, for fixed $t$, the weight $w_t$ that minimizes the variance of $y_\mathrm{MIX}$ is
\begin{equation} \label{eq:opt_w}
w_t^* = \frac{\mathbb{E}[\Vert y_\mathrm{LTSM} \Vert^2] - \mathbb{E}[y_\mathrm{DSM}^\top \, y_\mathrm{LTSM}]}{\mathbb{E}[\Vert y_\mathrm{DSM} \Vert^2] + \mathbb{E}[\Vert y_\mathrm{LTSM} \Vert^2] - 2 \mathbb{E}[y_\mathrm{DSM}^\top \, y_\mathrm{LTSM}]},
\end{equation}
where expectations are w.r.t.\ $(\theta_0, \theta_t, z) \sim p(\theta_0, z) p(\theta_t \,\vert \, \theta_0)$. In practice, we learn \(w_t=\sigma(\mathrm{MLP}(t))\) jointly with the score network minimizing the  loss from \cref{eq:loss-mix}, avoiding expectation estimates.\looseness=-1
\end{proposition}

\begin{proof}
We have $w_t^* = \arg\min_{w_t} \mathbb{E}[\Vert y_\mathrm{MIX}(w_t) \Vert^2]$. Expanding $y_\mathrm{MIX}$ in terms of $w_t$, $y_\mathrm{DSM}$ and $y_\mathrm{LTSM}$, differentiating w.r.t.\ $w_t$, and setting the result to zero yields \cref{eq:opt_w}.
\end{proof}

\section{Experiments} \label{sec:exps}

\begin{figure}[t]
    \centering
    \includegraphics[width=0.9\linewidth]{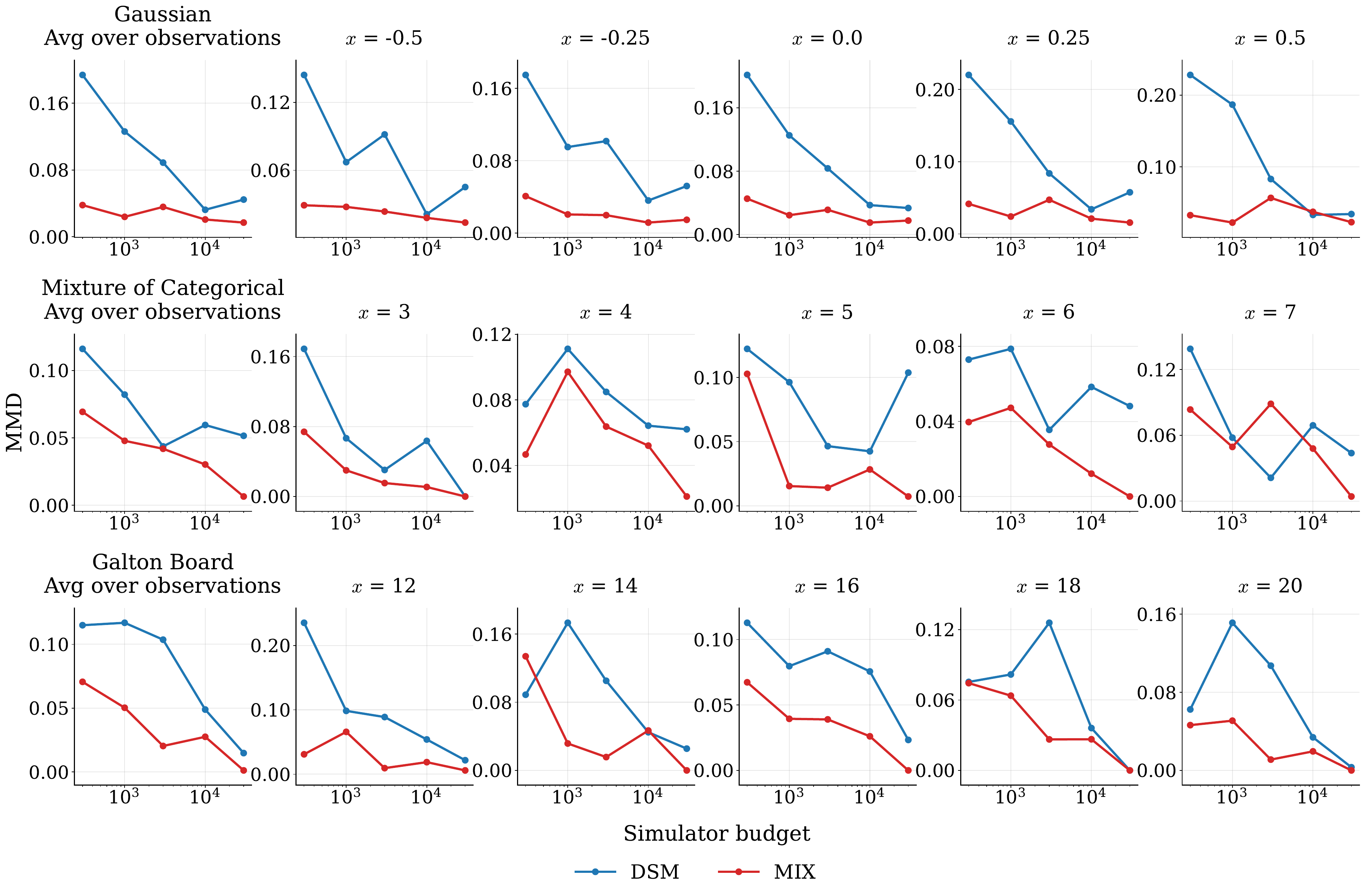}
    \caption{MMD (lower is better) vs.\ simulator budget (i.e.\ size of the dataset used for training) for each task (rows). Left column averages over five observations $x$; remaining columns show results for different potential values for the observations $x$. The Mixture improves over DSM, with the largest gap happening for the lower number of simulator calls.}
    \label{fig:mmd}
\end{figure}

We evaluate DSM (\cref{eq:dsm-loss}), LTSM (\cref{eq:loss-ltsm}), and the mixture approach (\cref{eq:loss-mix}) on three simulation-based inference (SBI) tasks. SBI is a powerful framework for inferring parameters \textbf{$\theta$} of complex simulators whose likelihood function $p(x|\theta)$ is intractable. Many such simulators are "gray-box" models that rely on internal latent variables \textbf{$z$} to generate observations \textbf{$x$}, exposing a tractable joint density $p(\theta,z,x)=p(\theta)\,p(z\mid\theta)\,p(x\mid\theta,z)$. Given an observation $x^\star$, the goal remains to find the posterior $p(\theta\mid x^\star)$ (see \cref{app:sbi_background} for a detailed background). Neural SBI methods-posterior (NPE), likelihood (NLE), and diffusion-based approaches \citep{papamakarios2016fast,lueckmann2017flexible,NEURIPS2018_2e9f978b,wood2010statistical,papamakarios2019sequential,pmlr-v96-lueckmann19a,geffner2023compositional}-train amortized models on simulator-generated datasets $\{(\theta_i,z_i,x_i)\}_{i=1}^M$, which can be used for any \(x\) at inference time. In practice, generating large datasets by calling the simulator is often expensive. This motivates a focus on sample efficiency: achieving strong performance from a limited number of simulator calls. While diffusion models trained with DSM have proven to be a powerful approach for SBI \citep{geffner2023compositional}, we study whether we can improve their sample efficiency by leveraging simulator's joint information via LTSM. We investigate this by training a score network $s_\psi(\theta_t, t, x)$ to approximate $\nabla_{\theta_t} \log p_t(\theta_t|x)$ and comparing its performance when supervised by the DSM target, the LTSM target, or their mixture.\looseness=-1



\textbf{Simulators.} We evaluate our approach on three simulators: a Gaussian model where scores can be computed exactly, a Mixture of Categoricals, and a Generalized Galton Board \citep{brehmer2020mining}. We briefly describe the models here, and provide full details in \cref{app:simulators}.

\emph{Gaussian:} $\theta\!\sim\!\mathcal{N}(0,1)$; $z\!\mid\!\theta\!\sim\!\mathcal{N}(\theta,1)$; $x\!\mid\!z\!\sim\!\mathcal{N}(z,1)$. By conjugacy $\theta\!\mid\!x\!\sim\!\mathcal{N}(x/3,\,2/3)$ and, diffusing only $\theta$, $\theta_t\!\mid\!x\!\sim\!\mathcal{N}(\alpha(t)\,x/3,\,1-\alpha^2(t)/3)$ with
$\nabla_{\theta_t}\log p_t(\theta_t\!\mid\!x)=\big(\alpha(t)\,x/3-\theta_t\big)/(1-\alpha^2(t)/3)$.
These closed forms give exact error/variance curves.

\emph{Mixture of Categorical:}
$p(\theta)=\mathcal{N}(0,1)$; $z\!\mid\!\theta\!\sim\!\mathrm{Bernoulli}(\sigma(\theta))$ with $\sigma(u)=1/(1+e^{-u})$; $x\!\mid\!z\!\sim\!\mathrm{Categorical}(\varphi^{(z)})$ over $K$ classes, $\varphi^{(0)}=\varphi_0,\ \varphi^{(1)}=\varphi_1\in\Delta^{K-1}$. 

\emph{Generalized Galton Board:}
$p(\theta)=\mathcal{N}(0,1)$; for $i=1,\dots,R$ (rows), $z_i\!\mid\!\theta\!\sim\!\mathrm{Bernoulli}(\sigma(\theta))$ with logistic $\sigma$; define $s_i=2z_i-1$ and
$x=\texttt{init\_pos}+\sum_{i=1}^{R}s_i$ with $\texttt{init\_pos}=\lfloor\texttt{num\_nails}/2\rfloor$.

\textbf{Evaluation Metrics and Setup.}
To compare methods, we measure three key diagnostics in order: (1) the conditional variance of the regression target $(y_{\mathrm{DSM}}^x,\;y_{\mathrm{LTSM}}^x,\;y_{\mathrm{mix}}^x)$ as a function of the diffusion time $t$; (2) the $\ell_1$ error between the learned score $s_\psi(\theta_t,t,x)$ and the true score $\nabla \log p_t(\theta_t \vert x)$ achieved by the different losses for the Gaussian model, where the true score can be computed analytically; and (3) the quality of the final posterior samples, assessed by the MMD \citep{papamakarios2019sequential,greenberg2019automatic,pmlr-v119-hermans20a} with a Gaussian kernel (\cref{app:kern_mmd}). Unless stated otherwise, all methods share the same network architecture, noise schedule, and training budgets to ensure a fair comparison.


\textbf{Results.} Our results show a consistent trade-off across all metrics. We begin the analysis with the regression-target variance (\cref{fig:variance}, left panel). As hypothesized, the DSM regression target's variance grows large as $t\to0$, whereas the LTSM target remains well-conditioned at low noise levels and only grows for larger $t$. The mixed regression target (with the optimal weight from \cref{eq:opt_w}) leverages the best of both worlds, combining their complementary strengths to maintain low variance across all noise levels. We report both the optimal weights $w_t^\ast$ and the trained $w_t$ as functions of $t$ in \cref{app:opt_weights}. This fundamental difference in variance directly translates to score estimation accuracy. On the Gaussian task, where the true score is known (\cref{fig:variance}, right panel), the Mixture (MIX) objective yields the approximation with the lowest overall error, by leveraging LTSM's accuracy at small $t$ and DSM's robustness at larger $t$. Ultimately, these improvements in score estimation yield higher-quality posterior samples, the main goal in SBI. As shown by the MMD results in \cref{fig:mmd}, the MIX method consistently generates better posterior samples than DSM across all tasks, and for several different potential observations $x^*$. The performance gains are most significant for smaller simulator budgets, demonstrating the improved sample efficiency of our approach. \looseness=-1




\clearpage

\bibliographystyle{plain}

\bibliography{references}

\clearpage
\appendix
\section{Experiment Details}
\label{app:exp}

\subsection{Simulation-based Inference}
\label{app:sbi_background}

Simulation-Based Inference (SBI) \citep{cranmer2020frontier,beaumont2019approximate}, also known as likelihood-free inference, is a class of methods for performing Bayesian inference on models where the likelihood function is intractable. In many scientific fields, mechanistic models are expressed as simulators parameterized by \textbf{$\theta$} that generate observations \textbf{$x$}. This defines an implicit likelihood $p(x|\theta)$ which can be sampled from by running the simulator, but cannot be evaluated analytically. The goal of SBI is to estimate the posterior distribution $p(\theta|x)$ given a real-world observation.

The intractability of the likelihood prevents the use of traditional inference algorithms like Markov Chain Monte Carlo \citep{metropolis1953equation,neal1993probabilistic} or Variational Inference \citep{jordan1999introduction,blei2017variational}, which require explicit likelihood evaluations. Neural SBI methods \citep{papamakarios2016fast,lueckmann2017flexible,NEURIPS2018_2e9f978b,wood2010statistical,papamakarios2019sequential,pmlr-v96-lueckmann19a,geffner2023compositional} overcome this by training a surrogate model (e.g., a normalizing flow \cite{papamakarios2021normalizing} or a diffusion model \citep{song2021scorebased}) to approximate the posterior, likelihood, or likelihood ratio, using a dataset of $(\theta, x)$ pairs generated by the simulator.

A crucial aspect of many simulators is their reliance on internal latent variables $z$ to produce an observation \cite{brehmer2020mining}. These "gray-box" simulators have a structure $p(\theta, z, x) = p(\theta) p(z|\theta) p(x|\theta, z)$. While the marginal likelihood $p(x|\theta)$ remains intractable (it would require marginalizing over the latent variables $z$), the full joint likelihood $p(x, \theta, z)$ is often tractable. This is the exact setting where our proposed method, LTSM, is applicable.

A classic example is the Galton Board (see \cref{fig:galton}):
\begin{itemize}
    \item The parameters $\theta$ define the configuration of the nails in each row.
    \item The latent variables $z$ represent the stochastic path of a dropped ball, a sequence of binary left/right bounces at each nail.
    \item The observation $x$ is the final bin where the ball lands.
\end{itemize}
The inference task is: given that a ball landed in bin $x$, what is the posterior distribution $p(\theta|x)$? The likelihood $p(x|\theta)$ is computationally expensive, as it requires summing the probabilities of all possible paths $z$ that end in bin $x$. However, for any single, complete path $z$, the joint probability $p(x, z, \theta)$ can be computed. LTSM is designed to exploit this accessible joint information to learn the posterior score of interest efficiently.

\subsection{DSM and LTSM for Posterior Estimation}
\label{app:dsm_ltsm_post_est}
All identities carry over when conditioning on an observation $x$. We train a conditional score network $s_\psi(\theta_t,t,x)$ and keep the forward kernel $p_t(\theta_t\mid\theta_0)$ unchanged. The DSM target remains
\[
y^{x}_{\mathrm{DSM}}(\theta_0,\theta_t,t)\;:=\;\nabla_{\theta_t}\log p_t(\theta_t\mid\theta_0),
\quad\text{with}\quad
\mathbb{E}\!\left[y^{x}_{\mathrm{DSM}}\mid \theta_t,x\right]=\nabla_{\theta_t}\log p_t(\theta_t\mid x).
\]
LTSM becomes
\[
y^{x}_{\mathrm{LTSM}}(\theta_0,z,t,x)\;:=\;\tfrac{1}{\alpha(t)}\,\nabla_{\theta_0}\log p(\theta_0,z\mid x)
\;=\;\tfrac{1}{\alpha(t)}\,\nabla_{\theta_0}\log p(\theta_0,z,x),
\]
and satisfies $\mathbb{E}\!\left[y^{x}_{\mathrm{LTSM}}\mid \theta_t,x\right]=\nabla_{\theta_t}\log p_t(\theta_t\mid x)$. Hence the mixed target is
\[
y^{x}_{\mathrm{mix}}(t)\;=\;(1-w_t)\,y^{x}_{\mathrm{LTSM}}+w_t\,y^{x}_{\mathrm{DSM}},
\]
and the loss integrates expectations over $(\theta_0,z,x)\sim p(\theta,z,x)$ and $\theta_t\sim p_t(\cdot\mid\theta_0)$.

\begin{figure}[t]
    \centering
    \includegraphics[width=0.5\linewidth]{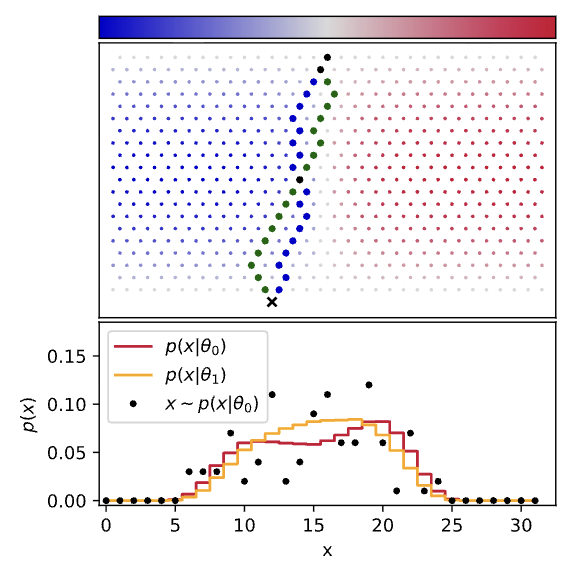}
    \caption{Toy Galton board depiction, extracted from Brehmer et al.\ \citep{brehmer2020mining}. The blue and green dark dots on the top figure represent potential paths of a ball dropped from the top. The positioning of the ``nails'' is defined by the parameters $\theta$, and the final position where the ball lands is denoted by $x$. For different parameters $\theta$ the distribution of positions for the ball's landing changes.}
    \label{fig:galton}
\end{figure}

\subsection{Simulator Details}
\label{app:simulators}

All simulators follow the SBI factorization
$p(\theta,z,x)=p(\theta)\,p(z\mid\theta)\,p(x\mid z,\theta)$.
We diffuse only $\theta$ under the VP–SDE and keep $z,x$ fixed. For LTSM we use the joint-score target
$y_{\mathrm{LTSM}}(\theta_0,z,t)=(1/\alpha(t))\,\nabla_{\theta_0}\log p(\theta_0,z,x)$,
where the joint gradient is obtained by forming the joint log-density and backpropagating to $\theta_0$ via autodiff.

\textbf{Gaussian.}
$\theta\sim\mathcal{N}(0,1)$,\; $z\mid\theta\sim\mathcal{N}(\theta,1)$,\; $x\mid z\sim\mathcal{N}(z,1)$.
Conjugacy gives $\theta\mid x\sim\mathcal{N}(x/3,\,2/3)$. Diffusing only $\theta$ yields
$\theta_t\mid x\sim\mathcal{N}(\alpha(t)\,x/3,\;1-\alpha^2(t)/3)$ and the closed-form posterior score
\[
\nabla_{\theta_t}\log p_t(\theta_t\mid x)=\frac{\alpha(t)\,x/3-\theta_t}{\,1-\alpha^2(t)/3\,}.
\]
These closed forms are used for exact score errors and label-variance.
The (analytic) joint gradient w.r.t.\ $\theta$ is
\[
\nabla_{\theta}\log p(\theta,z,x)= -\theta + (z-\theta).
\]

\textbf{Mixture of Categorical.}
$p(\theta)=\mathcal{N}(0,1)$;\;
$z\mid\theta\sim\mathrm{Bernoulli}(\sigma(\theta))$ with $\sigma(u)=1/(1+e^{-u})$;\;
$x\mid z\sim\mathrm{Categorical}(\varphi^{(z)})$ over $K$ classes.
We instantiate $\varphi^{(0)}=\varphi_0$ and $\varphi^{(1)}=\varphi_1$ once by sampling logits and applying a softmax; these $K$-dimensional probability vectors are then held fixed.
Because $x\perp\!\!\!\perp\theta\mid z$,
\[
\log p(\theta,z,x)=\log p(\theta)+\log p(z\mid\theta)+\log p(x\mid z)
\quad\Rightarrow\quad
\nabla_{\theta}\log p(\theta,z,x)= -\theta + \big(z-\sigma(\theta)\big).
\]
For posterior references we use rejection sampling over simulator draws, retaining $\theta$ when $x$ matches a target class.

\textbf{Generalized Galton Board.}
$p(\theta)=\mathcal{N}(0,1)$;\; for $i=1,\dots,R$ (number of rows),
$z_i\mid\theta\sim\mathrm{Bernoulli}(\sigma(\theta))$ with logistic $\sigma$.
Define steps $s_i=2z_i-1\in\{-1,+1\}$ and a deterministic readout
\[
x=\texttt{init\_pos}+\sum_{i=1}^{R} s_i,
\qquad
\texttt{init\_pos}=\big\lfloor \texttt{num\_nails}/2 \big\rfloor,
\]
where \texttt{num\_nails} is the number of bins/pegs across the board (so $x\in\{0,\ldots,\texttt{num\_nails}-1\}$).
Since $x$ depends on $z$ only, the joint gradient is
\[
\nabla_{\theta}\log p(\theta,z,x)= -\theta + \sum_{i=1}^{R}\big(z_i-\sigma(\theta)\big),
\]
computed via autodiff in practice. Posterior references $p(\theta\mid x^\star)$ are approximated by rejection sampling over simulator draws that hit the target bin $x^\star$.
\subsection{Kernel and MMD Details}
\label{app:kern_mmd}

\textbf{Kernel.}
All distributional comparisons are performed in $\theta$-space ($\theta\in\mathbb{R}^{d_\theta}$) using the Gaussian/RBF kernel
\[
k_\sigma(u,v)\;=\;\exp\!\Big(-\tfrac{\|u-v\|_2^2}{2\sigma^2}\Big).
\]
For each simulator and fixed observation $x^\star$, a single bandwidth $\sigma$ is selected once by the median-heuristic (median pairwise distance on a pilot set of reference posterior samples $\{\theta_i\!\sim p(\theta\mid x^\star)\}$) and then held fixed across methods and budgets.

\textbf{Population MMD.}
Given distributions $P$ and $Q$ on $\mathbb{R}^{d_\theta}$, the squared maximum mean discrepancy is
\[
\mathrm{MMD}^2_k(P,Q)
\;=\;
\mathbb{E}_{x,x'\sim P}k(x,x')
\;+\;
\mathbb{E}_{y,y'\sim Q}k(y,y')
\;-\;
2\,\mathbb{E}_{x\sim P,\;y\sim Q}k(x,y),
\]
which equals $0$ iff $P=Q$ for characteristic kernels such as the Gaussian.

\textbf{Estimator.}
For a fixed $x^\star$, let $\{\theta_i\}_{i=1}^m\!\sim p(\theta\mid x^\star)$ (reference) and
$\{\tilde\theta_j\}_{j=1}^n\!\sim \hat p(\theta\mid x^\star)$ (model). We use the standard unbiased U-statistic
\[
\widehat{\mathrm{MMD}}^2_u
=\frac{1}{m(m-1)}\!\!\sum_{i\neq i'}\! k_\sigma(\theta_i,\theta_{i'})
+\frac{1}{n(n-1)}\!\!\sum_{j\neq j'}\! k_\sigma(\tilde\theta_j,\tilde\theta_{j'})
-\frac{2}{mn}\sum_{i=1}^m\sum_{j=1}^n k_\sigma(\theta_i,\tilde\theta_j),
\]
and report $\widehat{\mathrm{MMD}}=\sqrt{\max\{\widehat{\mathrm{MMD}}^2_u,\,0\}}$ . The same kernel (and $\sigma$) is used for all methods and training budgets within each task.

\subsection{Mixture Weights: Optimal vs. Learned}
\label{app:opt_weights}
Recall the mixture target
\begin{equation}
y_{\mathrm{mix}}^x=(1-w_t)\,y_{\mathrm{LTSM}}^x+w_t\,y_{\mathrm{DSM}}^x,
\end{equation}
where $w_t\in[0,1]$ multiplies the DSM target ($w_t{=}0$ = pure LTSM; $w_t{=}1$ = pure DSM).
\Cref{prop:optimal-w} gives the variance–optimal coefficient
\begin{equation}
w_t^* = \frac{\mathbb{E}[\Vert y_\mathrm{LTSM} \Vert^2] - \mathbb{E}[y_\mathrm{DSM}^\top \, y_\mathrm{LTSM}]}{\mathbb{E}[\Vert y_\mathrm{DSM} \Vert^2] + \mathbb{E}[\Vert y_\mathrm{LTSM} \Vert^2] - 2 \mathbb{E}[y_\mathrm{DSM}^\top \, y_\mathrm{LTSM}]},
\end{equation}
with expectations over $(\theta_0,\theta_t,z)\sim p(\theta_0,z)\,p_t(\theta_t\mid\theta_0)$ at fixed $t$.
In \cref{fig:optimal_weights} we compute $w_t^\ast$ by Monte Carlo using large simulator-generated samples at each $t$ and clip to $[0,1]$.

\textbf{Learned Weights Used In Experiments.} 
For all results \emph{other than} the variance diagnostic, the mixture employs a learned schedule $w_t$.
We parameterize $w_t=\mathrm{MLP}(t)$ with a sigmoid output and learn it end-to-end jointly with the score network by minimizing the standard mixture score-matching loss from \cref{sec:ltsm} (no separate variance objective).
As shown in \cref{fig:trained_weights}, the learned schedules closely follow the analytic trend of $w_t^\ast$.

\begin{figure}[t]
    \centering
    \includegraphics[width=1.0\linewidth]{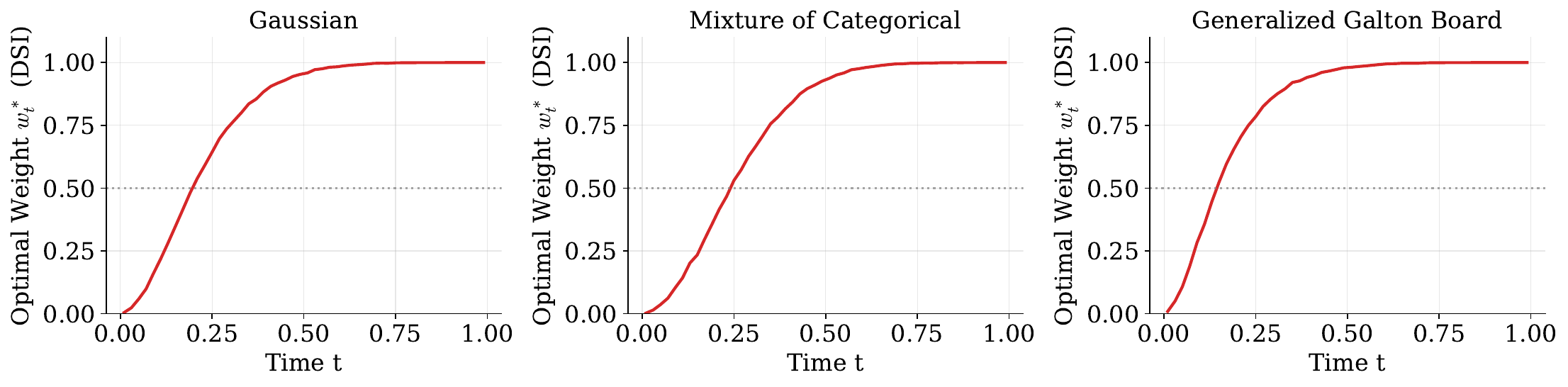}
    \caption{Variance-optimal DSM weight $w_t^\ast$ vs.\ time $t$.
    Computed via \Cref{prop:optimal-w}. Here $w_t^\ast$ is the coefficient on the DSM target in the mixture: $w_t^\ast{=}0$ means pure LTSM, $w_t^\ast{=}1$ means pure DSM. As expected, $w_t^\ast$ is low near $t\!\approx\!0$ and increases toward $1$ as $t\!\to\!1$.}
    
   \label{fig:optimal_weights}
\end{figure}
\begin{figure}[t]
    \centering
    \includegraphics[width=1.0\linewidth]{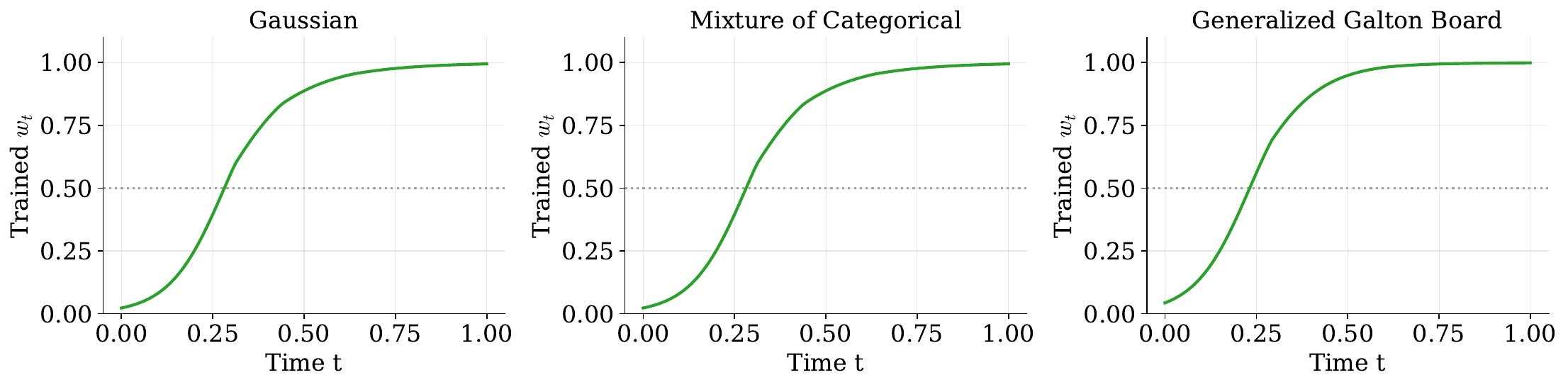}
\caption{Learned DSM weight $w_t$ vs.\ time $t$.
$w_t$ is the DSM coefficient in the mixture (0 = LTSM, 1 = DSM), trained jointly with the score network. It is low near $t\!\approx\!0$ and rises toward $1$ as $t\!\to\!1$, closely tracking $w_t^\ast$.}
   \label{fig:trained_weights}
\end{figure}
\section{Proofs}
\subsection{Standard Setting and Assumptions}
We diffuse only $\theta\in\mathbb{R}^{d_\theta}$ under the VP–SDE forward kernel
\[
\theta_t=\alpha(t)\,\theta_0+\sigma(t)\,\varepsilon_t,\qquad \varepsilon_t\!\perp\!(\theta_0,z),\quad \alpha(t)>0,
\]
so that $p_t(\theta_t\mid\theta_0)=q_t(\theta_t-\alpha(t)\theta_0)$ for some smooth noise density $q_t$.
Let $p_t(\theta_t)$ denote the marginal of $\theta_t$ when $(\theta_0,z)\sim p(\theta_0,z)$.

Also, the following assumption holds for both propositions. \emph{(A1)} $p(\theta_0,z)$ and $q_t$ are $C^1$ in their arguments and integrable; \emph{(A2)} differentiation may be interchanged with integration (dominated convergence); \emph{(A3)} the integration-by-parts (IBP) boundary term in $\theta_0$ vanishes.
\subsection{Latent Target Score Identity}
\textbf{Proposition 3.1} (Latent Target Score Identity, LTSI). \textit{Under the VP-SDE that diffuses $\theta$ and keeps $z$ fixed, the following identity holds:}
\begin{equation}
\nabla_{\theta_t} \log p_t(\theta_t) = \frac{1}{\alpha(t)} \mathbb{E}_{\theta_0, z \vert \theta_t}[\nabla_{\theta_0} \log p(\theta_0, z)].
\end{equation}
\begin{proof}
By independence,
\[
p_t(\theta_t)=\iint p(\theta_0,z)\,p_t(\theta_t\mid\theta_0)\,d\theta_0\,dz
=\iint p(\theta_0,z)\,q_t(\theta_t-\alpha\theta_0)\,d\theta_0\,dz,
\]
where $\alpha=\alpha(t)$ for brevity. Differentiating inside the integral (A2),
\[
\nabla_{\theta_t}p_t(\theta_t)
=\iint p(\theta_0,z)\,\nabla_{\theta_t}q_t(\theta_t-\alpha\theta_0)\,d\theta_0\,dz.
\]
By the chain rule with $u=\theta_t-\alpha\theta_0$,
\[
\nabla_{\theta_t}q_t(u)=\nabla_u q_t(u)
\qquad\text{and}\qquad
\nabla_{\theta_0}q_t(u)=(-\alpha)\,\nabla_u q_t(u),
\]
hence $\nabla_{\theta_t}q_t(\theta_t-\alpha\theta_0)=-(1/\alpha)\,\nabla_{\theta_0}q_t(\theta_t-\alpha\theta_0)$.
Substitute to get
\[
\nabla_{\theta_t}p_t(\theta_t)
=-\frac{1}{\alpha}\iint p(\theta_0,z)\,\nabla_{\theta_0}q_t(\theta_t-\alpha\theta_0)\,d\theta_0\,dz.
\]
Apply vector IBP in $\theta_0$ (A3): for each fixed $z,\theta_t$,
\[
\int p(\theta_0,z)\,\nabla_{\theta_0}q_t(\theta_t-\alpha\theta_0)\,d\theta_0
=\Big[p(\theta_0,z)\,q_t(\theta_t-\alpha\theta_0)\Big]_{\partial}-\int q_t(\theta_t-\alpha\theta_0)\,\nabla_{\theta_0}p(\theta_0,z)\,d\theta_0,
\]
and the boundary term $[\cdot]_\partial=0$ by (A3). Therefore
\[
\nabla_{\theta_t}p_t(\theta_t)
=\frac{1}{\alpha}\iint q_t(\theta_t-\alpha\theta_0)\,\nabla_{\theta_0}p(\theta_0,z)\,d\theta_0\,dz.
\]
Multiply and divide by $p_t(\theta_t)$ and also by $p(\theta_0,z)$ inside the integral:
\[
\begin{aligned}
\nabla_{\theta_t}\log p_t(\theta_t)
&=\frac{1}{p_t(\theta_t)}\,\nabla_{\theta_t}p_t(\theta_t)\\[2pt]
&=\frac{1}{\alpha}\iint
\underbrace{\frac{p(\theta_0,z)\,q_t(\theta_t-\alpha\theta_0)}{p_t(\theta_t)}}_{=\,p(\theta_0,z\mid\theta_t)}
\ \frac{\nabla_{\theta_0}p(\theta_0,z)}{p(\theta_0,z)}\ d\theta_0\,dz\\[2pt]
&=\frac{1}{\alpha}\ \mathbb{E}\!\left[\nabla_{\theta_0}\log p(\theta_0,z)\ \middle|\ \theta_t\right].
\end{aligned}
\]
\end{proof}
\subsection{Latent Target Score Matching}
\begin{proposition}[Latent Target Score Matching (LTSM)]
Let 
\(
y_{\mathrm{LTSM}}(\theta_0,z,t):=\alpha(t)^{-1}\nabla_{\theta_0}\log p(\theta_0,z)
\)
and $\eta(t)>0$.
Define
\[
\mathcal{L}_{\mathrm{LTSM}}(\psi)
=\int_0^1 \eta(t)\,
\mathbb{E}_{\,p(\theta_0,z)\,p_t(\theta_t\mid\theta_0)}
\big[\ \|s_\psi(\theta_t,t)-y_{\mathrm{LTSM}}(\theta_0,z,t)\|^2\ \big]\ dt.
\]
Then, for almost every $t\in(0,1]$, the unique $L^2$ minimizer is
\(
s_\psi^\star(\theta_t,t)=\nabla_{\theta_t}\log p_t(\theta_t).
\)
\end{proposition}

\begin{proof}
Let $y_t:=\alpha(t)^{-1}\nabla_{\theta_0}\log p(\theta_0,z)$ and write the objective as
\[
\mathcal{L}_{\mathrm{LTSM}}(\psi)=\int_0^1 \eta(t)\,\mathbb{E}\!\big[\ \|s_\psi(\theta_t,t)-y_t\|^2\ \big]\;dt,
\qquad \eta(t)>0.
\]
Fix $t$ and abbreviate $s(\cdot):=s_\psi(\cdot,t)$ and 
\(
m(\theta_t):=\mathbb{E}[\,y_t\mid \theta_t\,].
\)
Use the elementary identity
\[
\|s(\theta_t)-y_t\|^2
=\|s(\theta_t)-m(\theta_t)\|^2+\|y_t-m(\theta_t)\|^2
-2\langle s(\theta_t)-m(\theta_t),\,y_t-m(\theta_t)\rangle.
\]
Taking conditional expectation given $\theta_t$ kills the cross term because
$\mathbb{E}[\,y_t-m(\theta_t)\mid \theta_t\,]=0$. Hence
\[
\mathbb{E}\big[\|s(\theta_t)-y_t\|^2\big]
=\mathbb{E}\big[\|s(\theta_t)-m(\theta_t)\|^2\big]
+\mathbb{E}\big[\|y_t-m(\theta_t)\|^2\big].
\]
The second term does not depend on $s$, and the first is minimized uniquely when
$s(\theta_t)=m(\theta_t)$ (nonnegative with equality iff $s(\theta_t)=m(\theta_t)$ almost surely).
Therefore the pointwise minimizer at time $t$ is
\[
s^\star(\theta_t,t)=\mathbb{E}[\,y_t\mid \theta_t\,].
\]
By the Latent Target Score Identity (LTSI),
\[
\mathbb{E}[\,y_t\mid \theta_t\,]
=\frac{1}{\alpha(t)}\,\mathbb{E}\!\left[\nabla_{\theta_0}\log p(\theta_0,z)\,\middle|\,\theta_t\right]
=\nabla_{\theta_t}\log p_t(\theta_t).
\]
Since $\eta(t)>0$, integrating over $t$ preserves this minimizer for each $t$, which proves that
$s_\psi^\star(\theta_t,t)=\nabla_{\theta_t}\log p_t(\theta_t)$.
\end{proof}   
\clearpage

\end{document}